\documentclass[runningheads]{llncs}
\usepackage[margin=1.5in]{geometry}
\usepackage{amsmath}
\usepackage{comment}
\usepackage{dsfont}

\usepackage{booktabs} 
\usepackage{caption} 
\usepackage{enumitem}
\usepackage{subcaption} 
\usepackage{graphicx}
\usepackage{subcaption}
\usepackage{amssymb}
\usepackage{pgfplots}
\usepackage[all]{nowidow}
\usepackage[utf8]{inputenc}
\usepackage{tikz}
\usetikzlibrary{er,positioning,bayesnet}
\usepackage{multicol}
\usepackage{algpseudocode,algorithm,algorithmicx}
\usepackage{hyperref}
\DeclareUnicodeCharacter{2212}{-}

\definecolor{blue}{HTML}{1F77B4}
\definecolor{orange}{HTML}{FF7F0E}
\definecolor{green}{HTML}{2CA02C}

\pgfplotsset{compat=1.14}

\begin{document}
\title{On Generalization for Generative Flow Networks}

\author{\textbf{Anas Krichel} \hspace{0.3cm}
 \textbf{Nikolay Malkin} \hspace{0.3cm}
 \textbf{Salem Lahlou} \hspace{0.3cm}
 \textbf{Yoshua Bengio}}

\institute{
 Mila, Universit\'e de Montr\'eal \\
 \email{
 anas.krichel2@gmail.com,
 nikolay.malkin@mila.quebec,
 lahlosal@mila.quebec
 yoshua.bengio@mila.quebec
 }
}
\maketitle 
\begin{abstract}
Generative Flow Networks (GFlowNets) \cite{bengio2021flow,bengio2023foundations} have emerged as an innovative learning paradigm designed to address the challenge of sampling from an unnormalized probability distribution, called the reward function. This framework learns a policy on a constructed graph, which enables sampling from an approximation of the target probability distribution through successive steps of sampling from the learned policy. To achieve this, GFlowNets can be trained with various objectives \cite{malkin2022trajectory,bengio2023foundations,madan2023learning}, each of which can lead to the model's ultimate goal. The aspirational strength of GFlowNets lies in their potential to discern intricate patterns within the reward function and their capacity to generalize effectively to novel, unseen parts of the reward function. This paper attempts to formalize generalization in the context of GFlowNets, to link generalization with stability, and also to design experiments that assess the capacity of these models to uncover unseen parts of the reward function. The experiments will focus on length generalization -- generalization to states that can be constructed only by longer trajectories than those seen in training.

\end{abstract}

\section{Introduction}

Consider a reward function $R:\mathcal{X}\to\mathbb{R}^+$, where $\mathcal{X}$ is a finite set of objects that can be sequentially constructed, such as sets or graphs. $R$ assigns a positive real reward to each object in $\mathcal{X}$, and its normalization $x\mapsto\frac{R(x)}{Z}$ is the probability mass function of a distribution over $\mathcal{X}$. The objective of generative flow networks, or GFlowNets \cite{bengio2021flow,bengio2023foundations}, is to approximate this distribution by a composition of stochastic steps, where each step is follows a parametric conditional distribution that can be tractably sampled.

GFlowNets are amortized compound density estimators: they do not directly model $\frac{R(x)}{Z}$, but instead model sequences of actions which, through incremental modifications, construct elements of $\mathcal{X}$. The generative process is thus a policy conditioned on `intermediate states' -- entities not necessarily belonging to $\mathcal{X}$, but to a larger state space that may also contain `incomplete' objects. This policy models the transition probability from one intermediate state to the next (or to a terminal state, indicating that the construction is complete); equation \ref{eq:gfn_goal} formalizes these concepts. The (deterministic and episodic) Markov decision process in which the policy acts can be represented as a directed acyclic graph in which the nodes with no successors (terminal states) are identified with $\mathcal{X}$ and all other nodes are intermediate states.

The objects involved in GFlowNet training -- the state space, the fixed reward function, and the learned policy -- are the same as those in reinforcement learning \cite{sutton1998reinforcement,sutton1988learning,bertsekas1996neurodynamic}; however, the outcome of training is a policy that samples from the target distribution, rather than one that maximizes the expected terminal reward.\footnote{The precise connections between GFlowNets and entropy-regularized reinforcement learning were proved by \cite{tiapkin2024generative,deleu2024discrete}.} Once a policy has been trained to minimize a GFlowNet objective, new samples from the target distribution can be drawn by following the policy until a terminal state is reached.
Furthermore, training the policy as a parametrized model (such as a neural network) that takes a representation of an intermediate state as input and outputs a probability distribution over actions may lead to generalization to states not encountered during training.

GFlowNets can also be contrasted with non-amortized sampling algorithms, such as Markov Chain Monte Carlo (MCMC) methods. MCMC algorithms generate random samples by simulating a Markov process over the sample space $\mathcal{X}$ (or possibly some augmented space) whose stationary distribution is known to be the target distribution. With sufficient iterations, MCMC converges to the distribution of interest. However, MCMC exhibits limitations, particularly in scenarios where the modes of the reward function are separated by regions with low mass: the likelihood that a MCMC walker moves from one mode to another is exponentially small in the number of steps between the modes. In such instances, MCMC samples may become entangled in one mode, resulting in a lack of diversity in the generated objects. Moreover, MCMC techniques for discrete objects with combinatorial constraints are poorly developed relative to their continuous counterparts.

The challenge of mode collapse manifests in GFlowNets as well. A GFlowNet may become fixated on a particular mode during its training phase, motivating various off-policy exploration techniques to enhance discovery of modes during training.

\paragraph{Our contributions.}

This study attempts to rigorously formalize the concept of generalization within the framework of GFlowNets trained with the Trajectory Balance loss (Section \ref{sec:tb}). A key aspect of this research is the establishment, under strong assumptions, of probabilistic bounds that guarantee the generalization capabilities of GFlowNets. In addition, we introduce a definition of stability for GFlowNets, motivated by the link between stability and generalization \cite{bousquet2002stability}. This is further augmented by a formal proof of the stability characteristics of GFlowNets.

Moving beyond theoretical results, our work also includes an empirical validation to assess generalization. This is achieved by intentionally hiding various random segments of the reward function in a given environment, followed by an evaluation of the GFlowNet's ability to reconstruct these concealed components. A notable observation from this empirical investigation is that policies derived from the Detailed Balance loss exhibit a marginally superior capacity for generalization (reconstruction of the hidden parts of the reward) compared to those derived from Trajectory Balance loss. 

To provide a theoretical explanation for this observation, we propose a hypothesis that could pave the way to future experiments aimed at verifying its validity. This hypothesis, if substantiated, has the potential to significantly enhance our understanding of the mechanisms underpinning generalization in GFlowNets and could inform the development of more robust training methodologies.

\section{Preliminaries and definitions}

In alignment with the foundational work presented in \cite{bengio2023foundations}, this section aims to establish a formal framework for the environment in which the policy under consideration operates. We are given a Directed Acyclic Graph (DAG), denoted as $G = (\mathcal{S}, \mathcal{A})$, where $\mathcal{S}$ embodies the state space and $\mathcal{A}$ is a subset of $\mathcal{S}\times\mathcal{S}$ represents the action space. Within this context, a state in $\mathcal{S}$ is designated as the source $s_0$ if it lacks incoming edges. Conversely, a state is identified as a sink state $s_f$ if it possesses no outgoing edges. If both these types of states are present and unique, the graph $G$ is then classified as a pointed DAG. Additionally, all states that maintain a connection to the sink state are referred to as terminal states. We denote by $\mathcal{X}$ the set of terminal states.

A trajectory within the graph $G$ is defined as $\tau = (s_0 \rightarrow s_1 \rightarrow \dots \rightarrow s_{n-1} \rightarrow s_n)$, where $s_0$ is the source state and $s_n$ is a terminal state, while for each $i$ in the set $\{0, ..., n-1\}$, the transition $s_i \rightarrow s_{i+1}$ is an element of the action space $\mathcal{A}$. The set of all potential trajectories in this framework is denoted as $\mathcal{T}$. In this work, we assume, without loss of generality\footnote{In the general case, we can always add virtual nodes $\bar{s}$ between a terminal node $s$ and $s_f$ to satisfy this constraint}, that terminal states have exactly one outgoing edge: $\forall s \in \mathcal{X}, s\rightarrow s_f \in \mathcal{A}$. 

A positive function $\hat{P}_F(\cdot| s)$ that satisfies: $\forall$ $s \in S \setminus \{s_f\}$
$\sum_{s' \in \text{Child}(s)} \hat{P}_F(s' | s) = 1$ is called a forward transition policy, where $\text{Child}(s)$ is the children set of node $s$. Similarly a backward transition policy $\hat{P}_B(\cdot| s)$ is any positive function that verifies $\forall$ $s \in S \setminus \{s_0\}$
$\sum_{s' \in \text{Par}(s)} \hat{P}_B(s' | s) = 1$, where $\text{Par}(s)$ is the parent set of node $s$.

\subsection{Definition}

Given a DAG $G = (S, \mathcal{A})$, a reward function $R$ over $\mathcal{X}$ (the set of terminal states), the GFlowNet goal is to construct a forward transition policy $P_F\left(s_{i+1} \mid s_i ; \theta\right)$ such that:
\begin{equation}
 \forall x \in \mathcal{X} : \sum_{\tau \in \mathcal{T}: x \in \tau} \prod_{i=0}^{n-1} P_F\left(s_{i+1} \mid s_i ; \theta\right) = \frac{R(x)}{Z},
 \label{eq:gfn_goal}
\end{equation}
where $Z = \sum_{x \in \mathcal{X}} R(x)$. Meaning, that with such a learned $\hat{P}_F$, it suffices to sequentially sample following the policy, starting from $s_0$, until reaching a terminal state, to effectively obtain samples from the target probability distribution. Equation (\ref{eq:gfn_goal}) can not provide directly a useful loss to learn $\hat{P}_F$, since the sum is intractable and Z is unknown. The Trajectory balance loss \cite{malkin2022trajectory} for example provides a necessary and sufficient condition for a learned $\hat{P}_F$ to satisfy (\ref{eq:gfn_goal}), from which a loss function, optimizable with gradient descent, can be defined. However, it requires parametrizing a backward probability transition function $\hat{P}_B$, and learning an estimate $\hat{Z}$ (also denoted $Z_\theta$) of the partition function $Z$. 
In the GFlowNet's literature, a notion of ``flow", referring to positive quantities associated with edges or nodes, is used in other training losses. In this work, we focus on the Trajectory Balance loss, which does not require defining flow estimates.

\subsection{Trajectory Balance loss}
\label{sec:tb}

Let $G = (S, \mathcal{A})$ be a DAG, given a reward function $R$ over $\mathcal{X}$. For any state $s$ let $P_F(\cdot \mid s; \theta)$ and $P_B(\cdot \mid s'; \theta)$ be a forward and backward policy parametrized by $\theta$, respectively.
The trajectory loss for a given path $\tau = (s_0 \rightarrow s_1 \rightarrow \ldots \rightarrow s_n = x)$ is given by:
\begin{equation}
\mathcal{L}_{\rm TB}(\tau,P_{F}^{\theta}, P_{B}^{\theta}, Z_{\theta}, R) = \mathcal{L}_{\rm TB}(\tau, \theta) := \left( \log \frac{Z_{\theta} \prod_{t=1}^{n} P_F(s_t \mid s_{t-1}; \theta)}{R(x) \prod_{t=1}^{n} P_B(s_{t-1} \mid s_t; \theta)} \right)^2.
\end{equation}

As previously mentioned, when the value of the loss is minimal (zero), then $P_F^\theta$ satisfies equation (\ref{eq:gfn_goal}). The policy $P_B$ can either be fixed (e.g., $P_B(\cdot \mid s')$ can be the uniform distribution over the parents of $s'$), or learned alongside $P_F$ and $Z_\theta$. This is proven in \cite{malkin2022trajectory}.
\section{Statistical definition of generalization for GFlowNet}
Learning theory (LT) \cite{vapnik1995nature,bousquet2003introduction,valiant1984theory,shalevshwartz2014understanding} is a framework that formalize machine learning concepts and derive mathematical statements about learning algorithms. In learning theory, we take a step back and consider the theoretical properties of learning algorithms prior to the acquisition of empirical data. This allows to provide statistical guarantees for learning algorithms and may inform the design of new algorithms with desired theoretical properties. The goal of learning theory is to provide a formal framework to understand the process of learning and hopefully to develop algorithms that can learn from data with provable guarantees. However, research on LT has been mostly geared towards supervised learning. GFlowNets represent a new learning paradigm for which we are going to define relevant concepts.

We focus in the following definition, and on the remainder of the paper, on GFlowNets trained with TB using SGD (stochastic gradient descent), with one trajectory sampled on-policy (i.e., by following the currently learned $P_F^\theta$) per iteration.
\begin{definition}[A definition of generalization]
Let $R$\footnote{$R(\tau_i):=R(x_i)$ where $x_i$ is the terminal state of $\tau_i$.}
 be a reward function, we represent a GFlowNet by $(Z_{\theta}, P_B, P_F)$. Let $\tau_i \sim P_F(\cdot \mid \cdot;\theta_{i-1})$\footnote{$\tau \sim P_F(\cdot \mid \cdot;\theta)$ means $\tau \sim P$ where $\forall \tau = (s_1, \dots, s_n) \in \mathcal{T}, P(\tau) := \prod_{t=1}^{n-1} \hat{P}_F(s_{t+1} | s_t).
$}, $\theta_i = \theta_{i-1} - \gamma \nabla_{\theta}\mathcal{L}_{\rm TB}(\tau_i,\theta_{i-1})$ and $\theta_0$ is a fixed vector. 
Let $\mathcal{R}_n$ be defined by:

$$\mathcal{R}_n:= \frac{1}{n} \sum_{i=1}^{n} \mathcal{L}_{\rm TB}(\tau_i,R(\tau_i), \theta_n). 
$$
($\theta_n$ is the parameters of the GFlowNet, i.e., those of $P_F$, $Z$, and $P_B$ if it is learned).
If there exists $\pi$ a test distribution over $\mathcal{T}\times\mathbb{R^+}$, and a function $f$ \footnote{$f$ may depend on many other quantities than $\mathcal{R}_n,n$ and $\pi$.} such that \footnote{In general R is a deterministic function of $\tau$, the expectation could be taken over $\tau$ only. The more general case is to consider that R is random too. Proposition \ref{prop:bound_for_regression_assumption} treats R as a function of $\tau$.}
with high probability we have: 
$$ \displaystyle \mathop{\mathbb{E}}_{(\tau,R) \sim \pi}[ \mathcal{L}_{\rm TB}(\tau,R,\theta)] \le f(\mathcal{R}_n,n,\pi).
$$
then we say that the GFlowNet $(Z_{\theta}, P_B, P_F)$ $(f,\pi)$-generalizes.
More generally any inequality that bounds the expected loss over an unseen distribution of the TB loss is a form of measurement of generalization. Theoretically, any distribution over $\mathcal{T}$ that gives non-zero probabilities to all trajectories would serve as a test distribution. The bound is useful if $f$ decreases with both the increase of $n$ and decrease of $\mathcal{R}_n$.

\end{definition}

\begin{proposition}[Bound with i.i.d. assumption]
\label{prop:bound_for_regression_assumption}
We note for all $\tau$ in $\mathcal{T}$,
$\mathcal{L}_{\rm TB}(R, f_{\theta}(\tau)) := \mathcal{L}_{\rm TB}(\tau,P_{F}^{\theta}, P_{B}^{\theta}, Z_{\theta}, R) $
when $P_F$ and $P_B$ vary in a fixed hypothesis set (typically neural networks): $f_{\theta}(\tau) = \frac{Z_{\theta}\prod_{i=1}^{n}P_F(s_{i+1}|s_i;\theta)}{\prod_{i=1}^{n}P_B(s_{i}|s_{i-1};\theta)}$ describes another hypothesis set that we note $\mathcal{H}$. Let $\mathcal{G} = \{ (\tau, R)\mapsto \mathcal{L}_{\rm TB}(R,f_{\theta}(\tau)), f_{\theta} \in \mathcal{H} \} $. We suppose that:
$\exists M >0$, $ \forall \tau \in \mathcal{T}$, $\forall R, $ $ \forall f_{\theta} \in \mathcal{H}$, $\mathcal{L}_{\rm TB}(R, f_{\theta}(\tau))<M.$ Let $(\tau_1,R_1),...,(\tau_n,R_n) \overset{\mathrm{iid}}{\sim} \pi$ where $\pi$ is any distribution over $\mathcal{T}\times \mathbb{R}^+$
Then with probability 1-$\eta$ we have:
$$ \displaystyle \mathop{\mathbb{E}}_{(\tau,R) \sim \pi}[ \mathcal{L}_{\rm TB}(R,f_{\theta}(\tau)] \le1/n\sum_{i=1}^{n}\mathcal{L}_{\rm TB}(R_i,f_{\theta}(\tau_i)) + 2\mathcal{R(G)} + M\sqrt{\frac{\log(1/\eta)}{2n}},$$
$$ \displaystyle \mathop{\mathbb{E}}_{(\tau,R) \sim \pi}[ \mathcal{L}_{\rm TB}(R,f_{\theta}(\tau)] \le1/n\sum_{i=1}^{n}\mathcal{L}_{\rm TB}(R_i,f_{\theta}(\tau_i)) + M\sqrt{\frac{2d \log(m/d)}{n}} + M\sqrt{\frac{\log(1/\eta)}{2n}}, $$
where 
$\mathcal{R(G)}$ is the Rademacher complexity of $\mathcal{G}$ and $d$ is the pseudodimension of $\mathcal{G}$.

\end{proposition}
\begin{proof}[It is a direct application of the Rademacher properties for bounded losses see {\cite[p.~270]{mohri2012foundations} or \cite[p.~88]{bach}}]
\end{proof}
\begin{remark}
Assuming that the training trajectories are sampled independently from a fixed distribution is unrealistic. GFlowNet is not trained that way. Typically, it learns from trajectories sampled from a perturbed version of its actual policy. That motivates the next proposition. 
\end{remark}

\begin{proposition}[Beyond i.i.d. assumption]
Let $P_F^*$ be a global minimum of the trajectory balance loss $\mathcal{L}_{\rm TB}$, $P_F(\cdot|\cdot;\theta_n)$ the learned forward policy at the time n. $M_1$ and $M_2$ are respectively a lower and upper bound for the $\mathcal{L}_{\rm TB}$ loss ($M_1$ can be chosen to be 0). We assume that the constant Z is known. 

$$ \displaystyle \mathop{\mathbb{E}}_{\tau \sim P^*}[\mathcal{L}_{\rm TB}(\tau)] \le \frac{1}{\sqrt{2}}(M_1 + M_2)\displaystyle \mathop{\mathbb{E}}_{\tau \sim P_F(\theta_n)}[\mathcal{L}_{\rm TB}(\tau)]^{\frac{1}{4}} + \displaystyle \mathop{\mathbb{E}}_{\tau \sim P_F(\theta_n)}[\mathcal{L}_{\rm TB}(\tau)].$$
\end{proposition}
\begin{proof}
We use Lemma \ref{lma:tv_bound} (see the Appendix) with $P = P^* $,
$Q = P_F(\theta_n)$ we get:

$$\displaystyle \mathop{\mathbb{E}}_{\tau \sim P*}[\mathcal{L}_{\rm TB}(\tau, \theta_{n})] \le (M_1+M_2)||P^* - P_F(\theta_n) ||_{TV} + \displaystyle \mathop{\mathbb{E}}_{\tau \sim P_F(\theta_n)}[\mathcal{L}_{\rm TB}(\tau)]. $$
With the Jensen Inequality we have:
$$\displaystyle \mathop{\mathbb{E}}_{\tau \sim P_F(\theta_n)}[\mathcal{L}_{\rm TB}(\tau)] \ge \displaystyle \mathop{\mathbb{E}}_{\tau \sim P_F(\theta_n)}[\log(\frac{P_F(\theta_n)(\tau)}{P^*(\tau)})]^2.$$
Notice:
$$\displaystyle \mathop{\mathbb{E}}_{\tau \sim P_F(\theta_n)}[\log(\frac{P_F(\theta_n)(\tau)}{P^*(\tau)})] = KL(P_F(\theta_n)||P^*)
.$$

Then: 
\begin{equation}
KL(P_F(\theta_n)||P^*) \leq
\displaystyle \mathop{\mathbb{E}}_{\tau \sim P_F(\theta_n)}[\mathcal{L}_{\rm TB}(\tau)]^{1/2}.
\label{eq:jensen}
\end{equation}

Additionally, by Pinsker's inequality, we have the following: 
\begin{equation} \sqrt{\frac{1}{2}KL(P_F(\theta_n) || P^*)} \ge || P_F(\theta_n) - P^* ||_{TV}.
\label{eq:pinsker}
\end{equation}
From inequality (\ref{eq:jensen}) and (\ref{eq:pinsker}):
$$ || P_F(\theta_n) - P^* ||_{TV} \le \frac{1}{\sqrt{2}}\displaystyle \mathop{\mathbb{E}}_{\tau \sim P_F(\theta_n)}[\mathcal{L}_{\rm TB}(\tau)]^{\frac{1}{4}}.$$

And finally: 
$$ \displaystyle \mathop{\mathbb{E}}_{\tau \sim P^*}[\mathcal{L}_{\rm TB}(\tau)] \le \frac{1}{\sqrt{2}}(M_1 + M_2)\displaystyle \mathop{\mathbb{E}}_{\tau \sim P_F(\theta_n)}[\mathcal{L}_{\rm TB}(\tau)]^{\frac{1}{4}} + \displaystyle \mathop{\mathbb{E}}_{\tau \sim P_F(\theta_n)}[\mathcal{L}_{\rm TB}(\tau)].$$

\end{proof}

\section{Stability}
Measures of complexity are widely used in learning theory to assess statistical guarantees of learning algorithms \cite{vapnik1995nature}. 
Stability, on the other hand, holds a significant correlation with the concept of generalization, a relationship that we will elaborate on after a formal introduction of stability. Initially, stability is defined within the context of supervised learning algorithms\cite{bousquet2002stability,kearns1999algorithmic,wu2006learning,bendavid2006sober}, we recall its definition in the appendix, following \cite{bousquet2002stability,feldman2018generalization}. Inspired by that, we will extend the notion of stability to GFlowNets framework. 

\begin{definition}[Stability for GFlowNets]
Consider a DAG $G = (S, \mathcal{A})$. Let $R_1$ and $R_2$ two functions defined on the target space $\mathcal{X}$ (a subset of $S$) and take values in $\mathbb{R^+}$. 
Let $P_1$ and $P_2$ be respectively the distributions over trajectories induced \footnote{A GFlowNet can be seen as a mapping from the space of all reward functions 
${R:\mathcal{X} \longmapsto \mathbb{R^+}}$ to $\bigtriangleup(\mathcal{T})$ (set of distributions over trajectories)} by the forward policy learned by a GFlowNet trained with Trajectory Balance loss on $R_1$ and $R_2$.

We say that a GFlowNet is $\beta$-stable if:
for all reward functions $R_1$ and $R_2$ such that: $\exists \epsilon>0
, \forall x \in \mathcal{X} : |R_1(x) - R_2(x) |< \epsilon.$

then:
$$\forall \tau \in T, |P_1(\tau) - P_2(\tau) |< \beta\epsilon.$$
\end{definition}

Stability measures how much a GFlowNet is sensitive to small changes in the reward function. If a GFlowNet that was trained on $R_1$, drastically changes its policy to model a slightly perturbed version of $R_1$, it indicates that it is highly sensitive to minor changes in the reward function. This sensitivity suggests that the GFlowNet is unable to capture the underlying structure of the reward function.

\subsection{Result concerning GFlowNets} \label{sec:sampling}
\begin{proposition}\label{prop:stability}
Let $R_1$ and $R_2$ two reward functions defined on $\mathcal{X}$ and take values in $\mathbb{R^+}$ such that: 
\begin{align*}
 &\sum_{x \in \mathcal{X}}R_1(x) = \sum_{x \in \mathcal{X}}R_2(x), \\
 &\forall x \in \mathcal{X}, |R_1(x) - R_2(x) | < \epsilon.
\end{align*}

Let $P_1$ and $P_2$ be the learned probability distributions over trajectories by a GFlowNet that reach a global minimum of the Trajectory Balance loss and fix the backward transition probability to a uniform distribution. Then we have:
$$\exists \beta \in [0,1 ] , \ \forall \tau \in \mathcal{T}, |P_1(\tau) - P_2(\tau) |< \beta\epsilon. $$
\end{proposition}

In fact, the mapping from reward functions to trajectory flows that minimize the TB loss
\[\left(R(x)\right)_{x\in\mathcal{X}}\mapsto\left(Z\cdot P(\tau)\right)_{\tau\in\mathcal{T}},\quad Z:=\sum_{x\in\mathcal{X}}R(x),\]
is linear, and one can take $Z\beta$ to be its $L^1$-operator norm (equivalently, its Lipschitz constant). In the proof below, we give an explicit construction of the constant $\beta$.

\begin{proof}

Let $P_1$ and $P_2$ be two probability distributions over the set of trajectories that verify the Proposition \ref{prop:stability} assumptions. We have: 
$$P_1(\tau) =\prod_{i=0}^{n-1} \hat{P}_F^1\left(s_{i+1} \mid s_i ; \theta_1\right), P_2(\tau) =\prod_{i=0}^{n-1} \hat{P}_F^2\left(s_{i+1} \mid s_i ; \theta_2\right).$$

Following the assumptions, we choose $P_B^1$, $P_B^2$ to be uniform, formally: 
$$P_B^1(\tau)=P_B^2(\tau)=\prod_{i=1}^{n} \frac{1}{\left|\text{Par}(s_i)\right|}.$$
We have: 
$$\forall \tau \in \mathcal{T}, \mathcal{L}_{T B}\left(\tau, \hat{P}_F^1, \hat{P}_B^1, Z_1\right)= \mathcal{L}_{T B}\left(\tau, \hat{P}_F^2, \hat{P}_B^2, Z_2\right)=0.$$
We define: 
$$A_1 :=\prod_{i=1}^{n+1} \hat{P}_B^1\left(s_{i-1} \mid s_i ; \theta_1\right), A_2 :=\prod_{i=1}^{n+1} \hat{P}_B^2\left(s_{i-1} \mid s_i ; \theta_2\right).$$
We have $A_1= A_2$ from the uniform hypothesis.

Let $\tau=(s_0,...,s_n) \in \mathcal{T}$, $C=\max_{\tau \in \mathcal{T}}\prod_{i=1}^{n+1} \hat{P}_B^1\left(s_{i-1} \mid s_i ; \theta_1\right)$and $Z:=Z_1=Z_2$, then:

\begin{align*}
|P_1(\tau) &- P_2(\tau) | = |\prod_{i=0}^{n-1} \hat{P}_F^1\left(s_{i+1} \mid s_i ; \theta_1\right)-\prod_{i=0}^{n-1} \hat{P}_F^2\left(s_{i+1} \mid s_i ; \theta_2\right)|\\
&= \left|\frac{Z_1Z_2\prod_{i=0}^{n-1} \hat{P}_F^1\left(s_{i+1} \mid s_i ; \theta_1\right)}{A_1A_2R_1(s_n)R_2(s_n)}-\frac{Z_1Z_2\prod_{i=0}^{n-1} \hat{P}_F^2\left(s_{i+1} \mid s_i ; \theta_2\right)}{A_1A_2R_1(s_n)R_2(s_n)}\right|\frac{A_1A_2R_1(s_n)R_2(s_n)}{Z_1Z_2}\\
&= \left|\frac{Z_2}{A_2R_2(s_n)}-\frac{Z_2}{A_1R_1(s_n)}\right|\frac{A_1A_2R_1(s_n)R_2(s_n)}{Z_1Z_2}\\
&\le C\left|\frac{R_1(s_n)}{Z_1}-\frac{R_2(s_n)}{Z_2}\right|\\
&\le \frac{C}{Z}|R_1(s_n)-R_2(s_n)|\\
&\le \frac{C}{Z}\epsilon.
\end{align*}

\end{proof}

\section{Experiments}
We consider the Hypergid toy environment initially studied in \cite{bengio2021flow} to investigate the generalization properties of GFlowNets. Let $S = \left\{ [a, b] \mid a, b \in \mathbb{Z}, 0 \leq a, b \le N \right\}$ be the state space, we define a reward function over $S$ by: $\forall x \in S: R(x) = 10^{-3} + \sum_{i=1}^{9}\mathds{1}_{A_{i}}(x).$ Where the $A_i$ are regions where the reward is the higher. The allowed actions are for each state is either increment one of the coordinates by +1 or to terminate. This defines a DAG and is called the 2-D grid environment. All states are terminal and $\mathcal{S} = \mathcal{X}$. We train a GFlowNet on a 9-mode reward function, deliberately concealing random chosen states of the grid—meaning we don't compute the TB loss for any trajectory that contains those chosen states (see Figure~\ref{fig:hiding_states}). At each training iteration, we show in Figure~\ref{fig:jsd} the Jensen–Shannon divergence between the true distribution and the learned distribution while not learning from trajectories containing the hidden states (the learned distribution is computed exactly without any empirical approximation). The results are averaged over 7 seeds.
This process was carried out for three distinct scenarios: TB loss, DB loss, and the FL-DB parametrization \cite{pan2023better}.

\begin{figure}[!htbp] 
 \centering
 
 \begin{subfigure}[t]{0.48\textwidth}
 \centering
 \includegraphics[width=\textwidth]{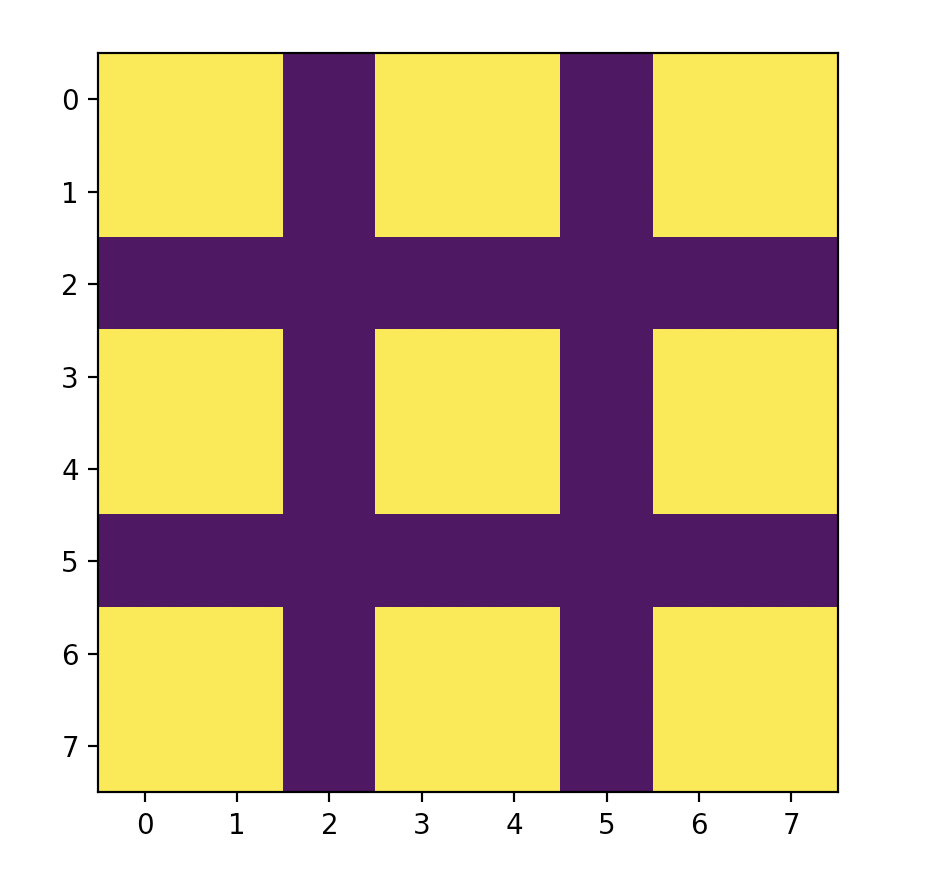}
 \caption{The reward calculated over each state of the 8 by 8 grid environment}
 \end{subfigure}
 \hfill 
 \begin{subfigure}[t]{0.48\textwidth}
 \centering
 \includegraphics[width=\textwidth]{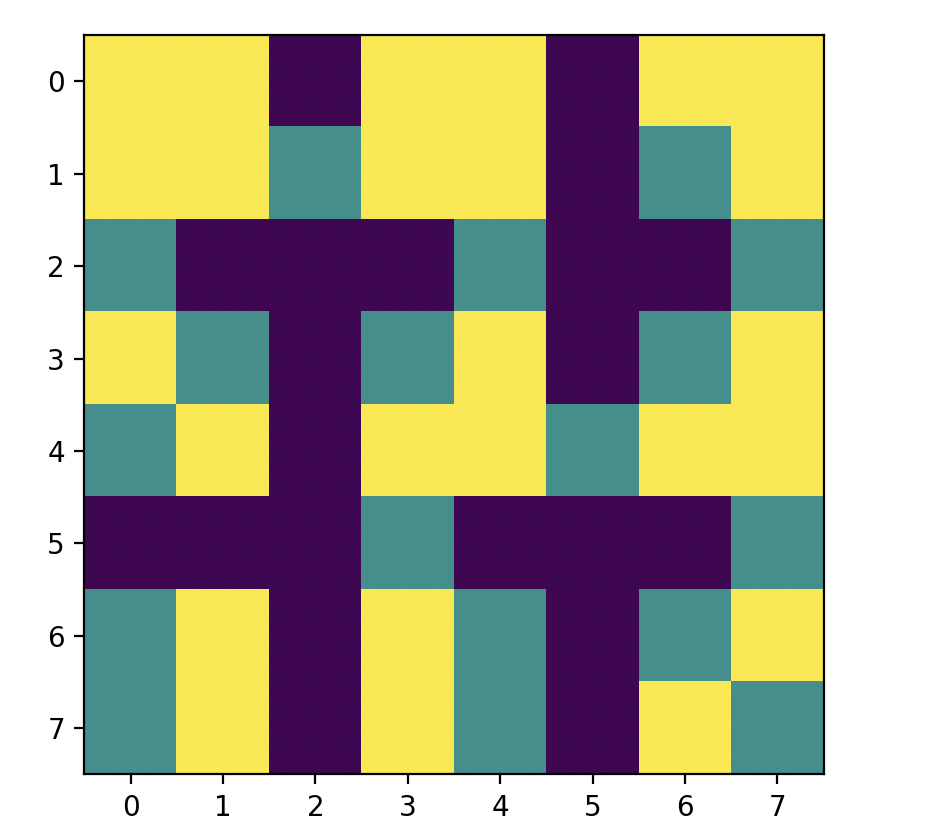}
 \caption{During training we intently don’t allow the agent to terminate in any state with green color}
 \end{subfigure}
 \caption{Illustration of the process of hiding states in order to measure generalization}
 \label{fig:hiding_states}
\end{figure}

\subsection{How we measure generalization}
Let $P^T_{\theta}$ be the learned terminal\footnote{A terminal probability is given by: $$\forall x \in \mathcal{X}: P^{T}_{\theta}(x) =\sum_{\tau \in \mathcal{T}: x \in \tau} \prod_{i=0}^{n-1} \hat{P}_F\left(s_{i+1} \mid s_i ; \theta\right).$$ It can be computed recurvisely using equation (33) in \cite{malkin2023gflownets}.} distribution and $P_{*}^{T} := \frac{R}{Z}$ the normalized reward (true distribution). The Jensen–Shannon divergence
is defined by:
$$D_{\text{JS}}(P^T_{\theta} \parallel P_{*}^{T}) = \frac{1}{2} \left( D_{\text{KL}}(P^T_{\theta} \parallel M) + D_{\text{KL}}(P_{*}^{T} \parallel M) \right),$$
where
$$M = \frac{1}{2}(P^T_{\theta} + P_{*}^{T}).$$
This divergence serves as our loss function, capturing the dissimilarity between the learned terminal distribution and the target normalized reward distribution $P_{*}^{T}$. To assess the learning dynamic of each algorithm, we track the Jensen-Shannon Divergence (JSD) loss at each training iteration. The curves labeled FL-DB, DB, and TB illustrate the JSD loss, while hiding pre-selected random states. The curves FL-D$B_{whm}$, D$B_{whm}$, and T$B_{whm}$ demonstrate the JSD loss progression during training iterations where state suppression was not employed.

\begin{figure}[!htbp]
\centering
\begin{minipage}{\textwidth}
\centering
 \begin{subfigure}[t]{0.48\textwidth}
 \centering
 \includegraphics[width=\textwidth]{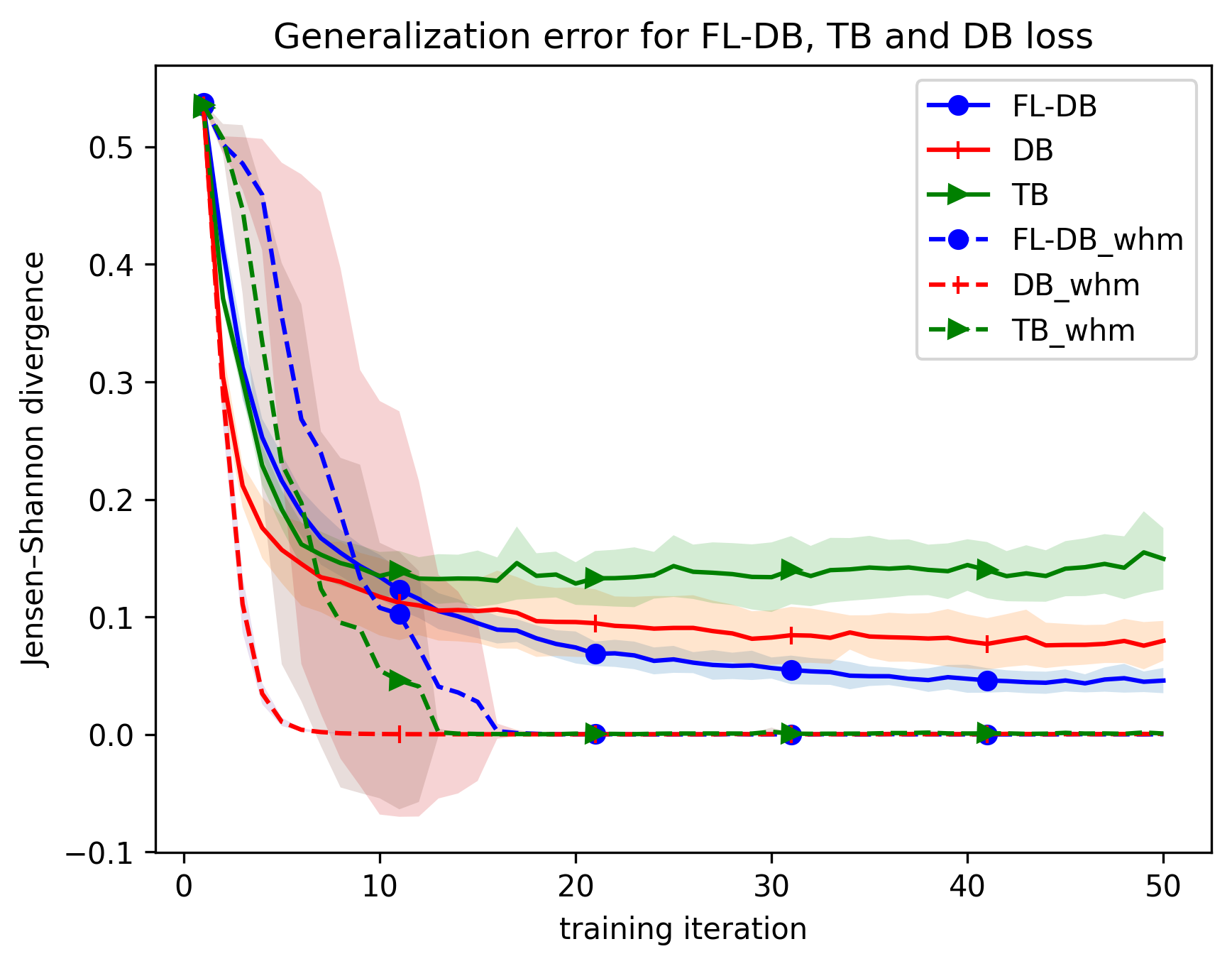}
 \caption{The grid is of size 20 by 20 and we hide 302 randomly chosen states}
 \label{fig:image4}
 \end{subfigure}
 \hfill
 \begin{subfigure}[t]{0.48\textwidth}
 \centering
 \includegraphics[width=\textwidth]{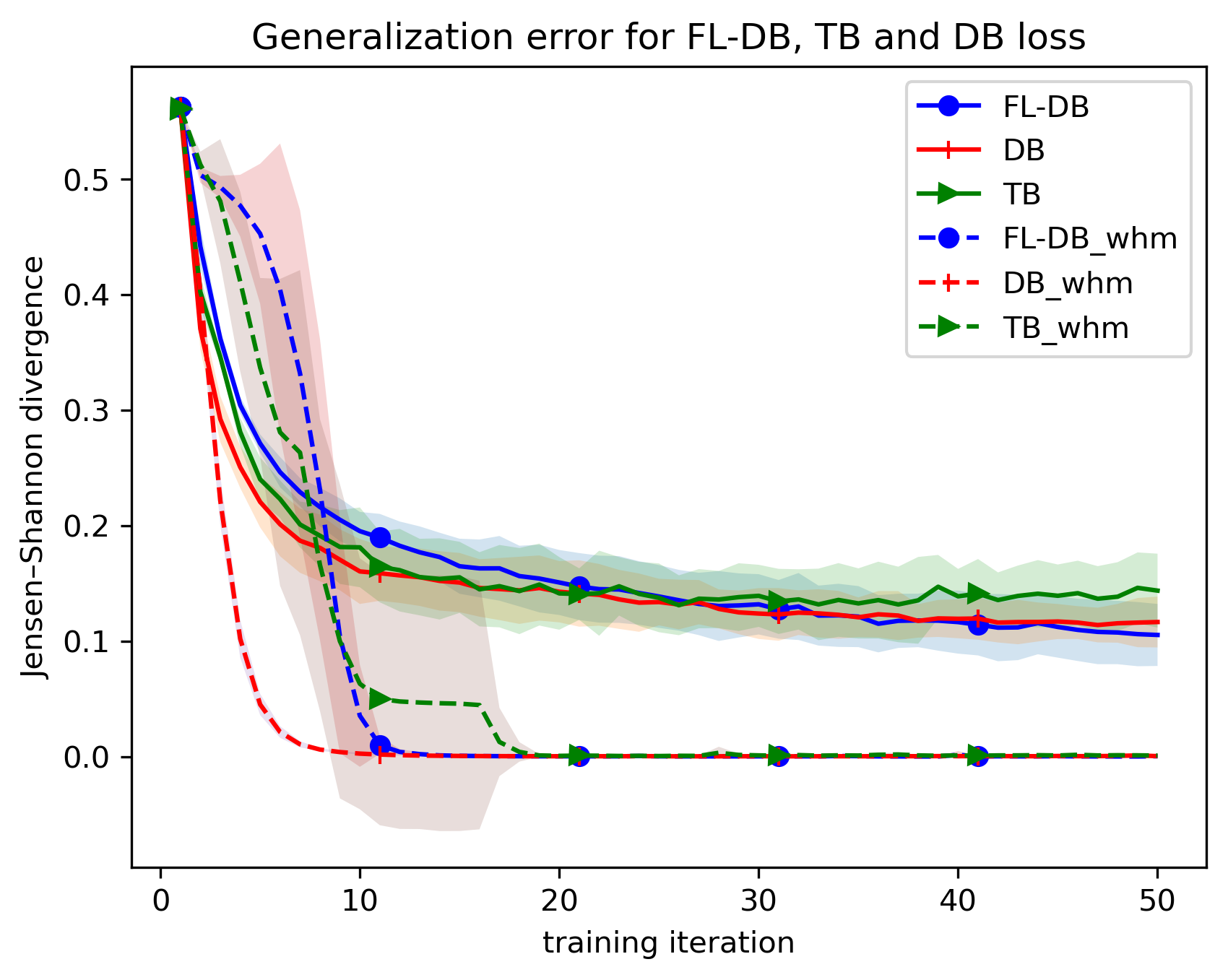}
 \caption{The grid is of size 30 by 30 and we hide 675 randomly chosen states}
 \label{fig:image5}
 \end{subfigure}
 \hfill
 \begin{subfigure}[t]{0.48\textwidth}
 \centering
 \includegraphics[width=\textwidth]{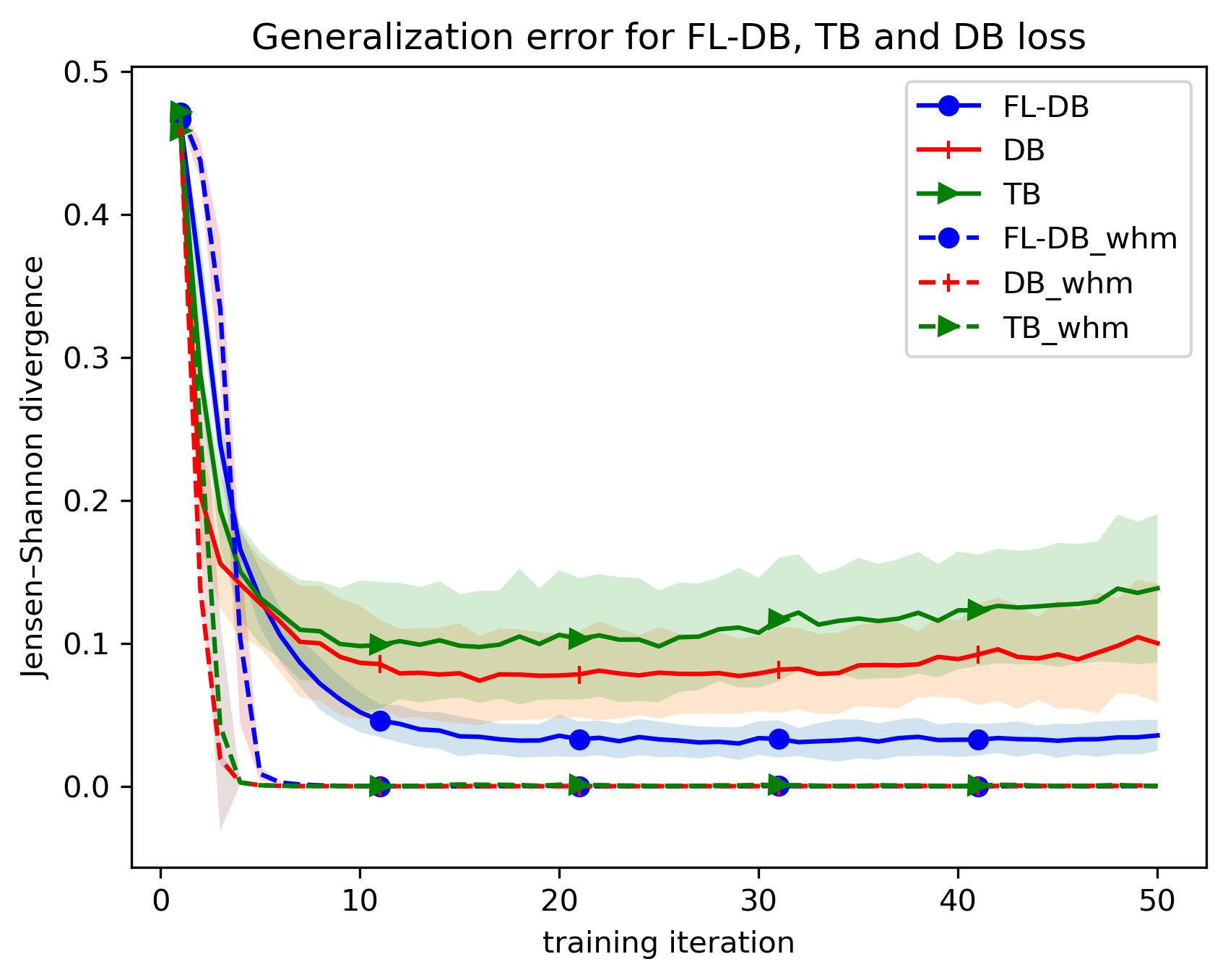}
 \caption{The grid is of size 12 by 12 and we hide 108 randomly chosen states}
 \label{fig:image5-2}
 \end{subfigure}
\end{minipage}
\vspace{1em} 

\begin{minipage}{\textwidth}
 \centering
 \begin{subfigure}[t]{0.48\textwidth}
 \centering
 \includegraphics[width=\textwidth]{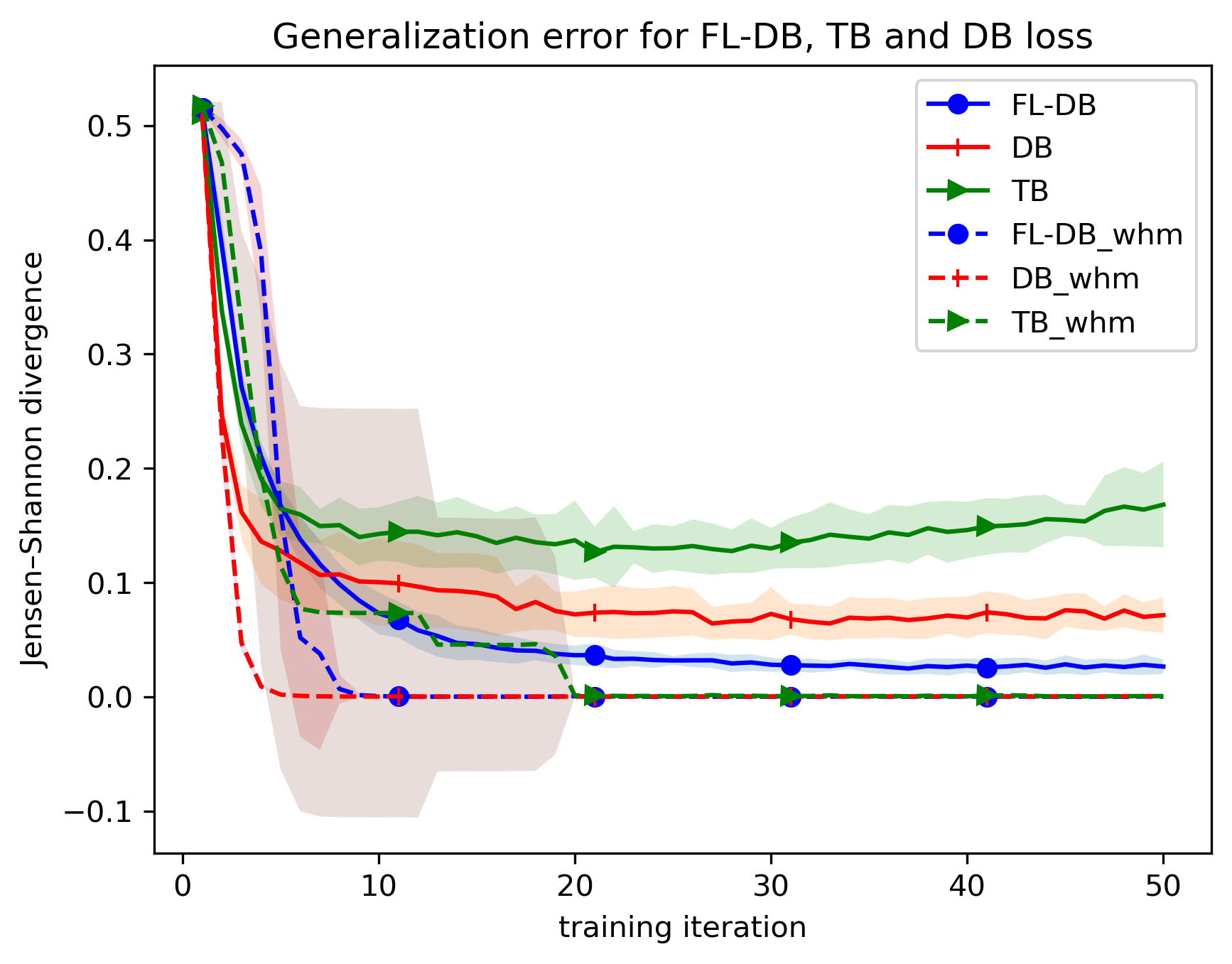}
 \caption{The grid is of size 15 by 15 and we hide 170 randomly chosen states}
 \label{fig:image4-2}
 \end{subfigure}
 \hfill
 \begin{subfigure}[t]{0.48\textwidth}
 \centering
 \includegraphics[width=\textwidth]{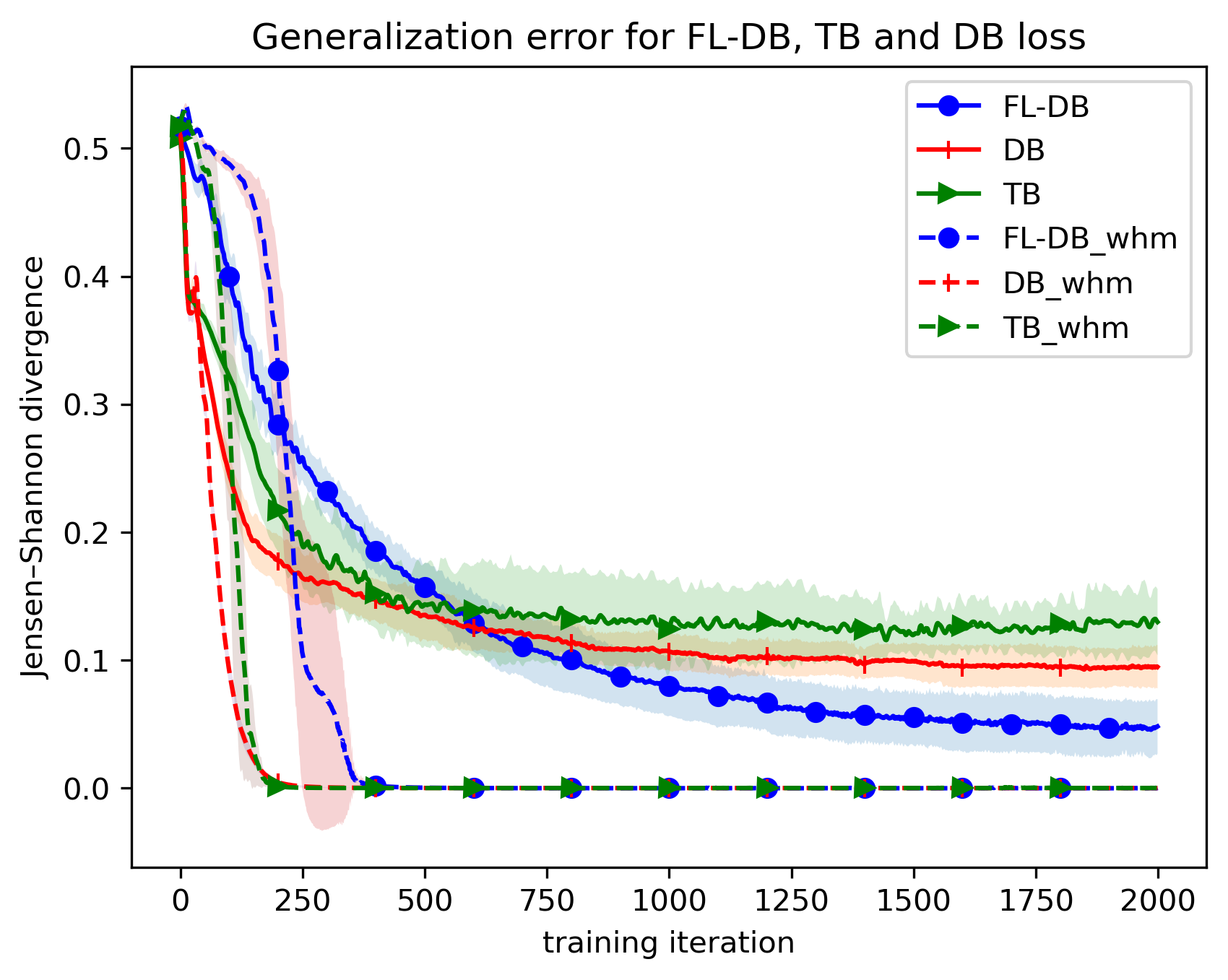}
 \caption{The grid is of size 8 by 8 and we hide 48 randomly chosen states}
 \label{fig:image5-3}
 
 \end{subfigure}
\end{minipage}

\caption{Tracking generalization of different losses and environments sizes.}
\label{fig:jsd}
\end{figure}

\subsection{Observation}
In our analysis, a discernible pattern emerged among the evaluated algorithms. Notably, the DB algorithm exhibited a superior capacity for generalization compared to the TB algorithm, as evidenced by the data presented in Figure~\ref{fig:jsd}. This trend suggests that the DB algorithm's proficiency in learning the state flow significantly contributes to its enhanced generalization abilities. It is important to note that the FL-DB variant, which has access to rewards for intermediate states (meaning the state flow is offset to the intermediate reward plus a learnable quantity)\footnote{$\log(F(s)) = \log(R(s)) + \log(\hat{F}(s,\theta))$} consistently outperformed other loss functions in our evaluations. However, it is crucial to recognize that access to intermediate rewards is not typically available in practical scenarios, thereby limiting the general applicability of the FL-DB approach in real-world settings.
\subsection{Examples of the learned distribution}
Presented in Figure~\ref{fig:ood} are the visual representations of the sampling distribution obtained from the trained GFlowNet which was prevented from visiting the modes below the anti-diagonal in the grid environment. 
These figures show the exact learned distribution over terminal states after training. 
\begin{figure}[!htbp] 
 \centering
 
 \begin{subfigure}[t]{0.48\textwidth}
 \centering
 \includegraphics[width=\textwidth]{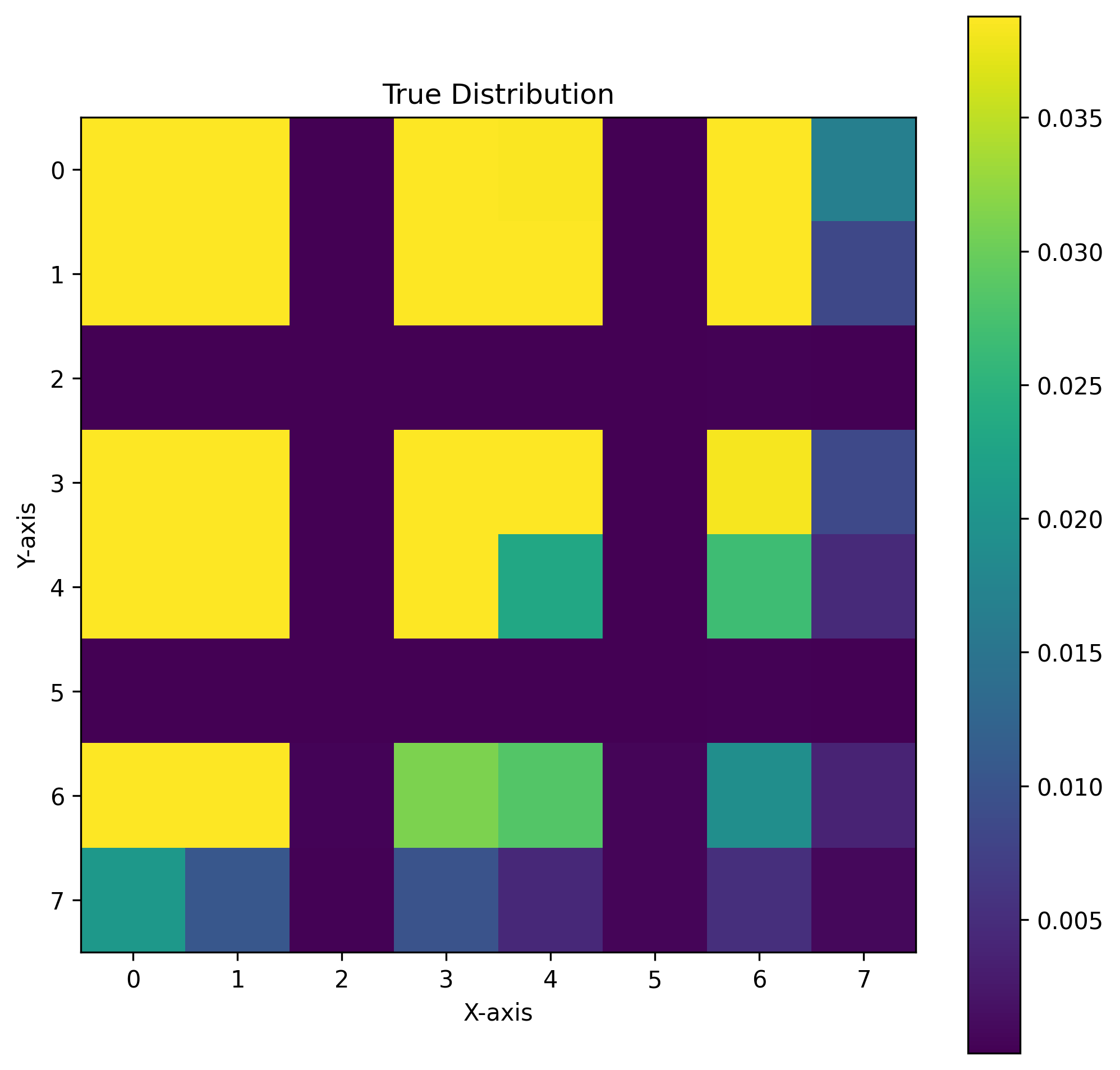}
 \caption{During training we intently don’t allow the agent to terminate in any state that has length more than 7 using DB loss. This image is the learned distribution after training}
 \label{fig:tb2}
 \end{subfigure}
 \hfill 
 \begin{subfigure}[t]{0.48\textwidth}
 \centering
 \includegraphics[width=\textwidth]{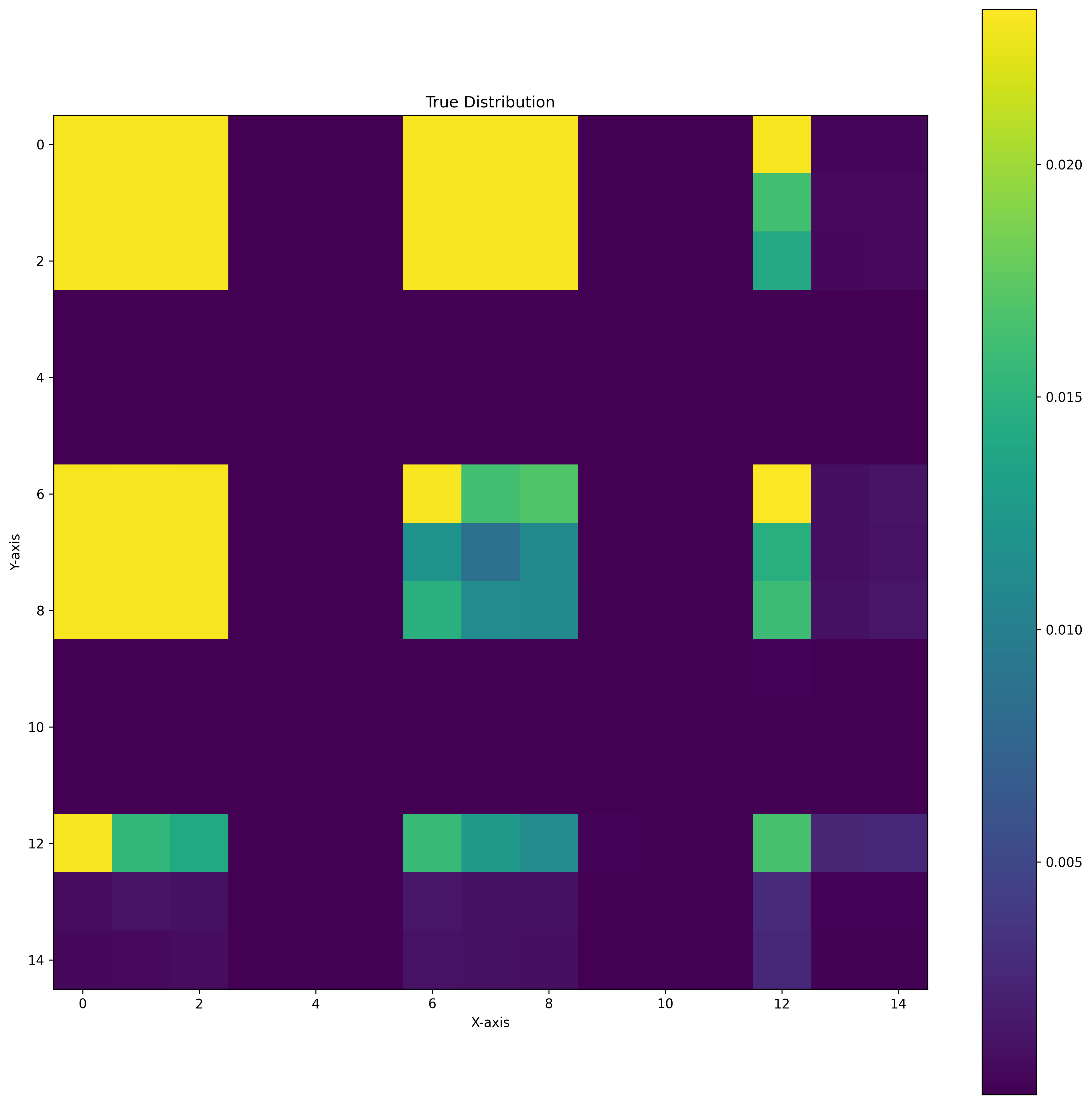}
 \caption{During training we intently don’t allow the agent to terminate in any state that has length more than 12 using DB loss. This image is the learned distribution after training}
 \label{fig:copy5-2}
 \end{subfigure}
 \caption{Showing the challenge of out-of-distribution generalization}
 \label{fig:ood}
\end{figure}

\subsection{Limits and future work}
Figure~\ref{fig:tb2} shows a capacity to reconstruct all the unseen anti-diagonal and perfectly localizing the modes in the right places. However in Figure~\ref{fig:copy5-2} the learned distribution demonstrates a weaker ability to generalize to states having far representation to the training states (specifically, states having in their coordinates the values 14 or 15). This limitation may be attributable to the requirement for a degree of out-of-distribution generalization \cite{arjovsky2019invariant,bengio2011deep,haley1992extrapolation,krueger2021ood,balestriero2021learning,barrett2018measuring,zhan2022evaluating} capacity to generalize to states sampled from a distribution different from the training one \cite{liu2023towards,ganin2016domain,rosca2019understanding,krueger2021ood}—a capacity that current learning algorithms manifest partially. Further experimental investigations are warranted to substantiate the hypothesis that Detailed Balance (DB) contributes to enhanced generalization capabilities. To robustly validate this hypothesis, it is proposed that these experiments be conducted within a range of more complex and challenging environments. The stability section may be extended by a theorem for non-uniform $P_B$.

\bibliographystyle{splncs04}
\bibliography{biblio}
\appendix

\section{Appendix}

\begin{lemma}
\label{lma:tv_bound}
Let Q and P be two probability measures on $(\Omega,\mathcal{F})$, let $h:\Omega \longmapsto \mathbb{R}$ be a Borel-measurable function such that $-M_1 \le h \le M_2$ for some $M_1,M_2\ge0$ Then:
$$ |\mathbb{E}_{Q}[h] - \mathbb{E}_{P}[h] | \le (M_1 + M_2)||P-Q||_{TV}. $$
\end{lemma}
\begin{proof}[see {\cite[p.~5]{kuznetsov2016generalization}}]
\end{proof}

\subsection{Stability in supervised learning} \label{subsec:statistical-summaries}
Let $\mathcal{X}$ be a set that represents the input space. Let $\mathcal{Y}$ be a set that represent the target space. $\mathcal{H}$ is a set of functions from $\mathcal{X}$ to $\mathcal{Y}$. Each function $h \in \mathcal{H}$ is called a hypothesis. A learning algorithm is a mapping of the form $l_A:\bigcup_{n \in \mathbb{N}}(\mathcal{X}\times\mathcal{Y})^{n} \longmapsto \mathcal{H}$. Some learning algorithms may output different hypotheses for the same data set, this randomness may be modeled by an additional input set $\Omega$ provided with a probability measure $P$. We limit ourselves to deterministic learning algorithms. A loss function is any mapping of the form $L: \mathcal{X}\times\mathcal{Y} \to \mathbb{R}^+$.
\begin{definition}[Uniform stability]
Let $S_1$ and $S_2$ be two elements of S that differ with a single point. $l_A(S_1)$ and $l_A(S_2)$ the two learnt hypotheses by the learning algorithm $l_A$. We say that $l_A$ is a $\beta$-uniformly stable learning algorithm if $ \beta$ is the smallest constant that satisfies:

$$ \forall (x,y) \in \mathcal{X}\times\mathcal{Y}, |L(l_A(S_1)(x), y) - L(l_A(S_2)(x), y) |< \beta.$$
\end{definition}

\subsection{Remark}

To see the similarity between stability in supervised learning and the one for GFlowNets, $S_1$ and $S_2$ are transformed to $R_1$ and $R_2$. 
$L(l_A(S_1)(x),y)$ and $L(l_A(S_2)(x),y)$ are replaced respectively by $P_1$ and $P_2$, probability distributions over trajectories given by a GFlowNet
that reach a global minimum using a Trajectory Balance loss and fix the backward
transition probability to a uniform distribution, trained respectively on $R_1$ and $R_2$.

\end{document}